\documentclass{article}

\PassOptionsToPackage{round}{natbib}
\usepackage[preprint]{neurips_2026}

\usepackage[utf8]{inputenc}
\usepackage[T1]{fontenc}
\usepackage{hyperref}
\usepackage{url}
\usepackage{amsmath,amssymb,amsfonts,amsthm}
\usepackage{booktabs}
\usepackage{microtype}
\usepackage{graphicx}
\usepackage{nicefrac}
\usepackage{cleveref}

\newtheorem{theorem}{Theorem}
\newtheorem{proposition}[theorem]{Proposition}

\newcommand{\R}{\mathbb{R}}

\newcommand{\Od}{\mathrm{O}(d)}
\newcommand{\Rot}{R_{\mathrm{pos}}}

\begin{document}

\title{Content-Based Addressing for Long Context}
\author{%
  Mahesh Godavarti \\
  A Carrot, Inc. \\
  \texttt{m@acarrot.com}
}
\maketitle

\begin{abstract}
Rotary position embedding (RoPE) uses each token's integer position to determine the rotation applied inside attention.
This works well for local token order, but increasing context length creates a positional train-test mismatch: RoPE produces relative rotations at offsets not seen during training.
Methods that rescale, interpolate, randomize, or bias positions specify how attention handles those offsets, but still derive positional information from a growing token counter.
We instead divide a token stream into units, retain ordinary RoPE positions within each unit, and assign every completed unit an address computed from its content.
Adding units then applies the same learned map to new content rather than extending a positional range or an identifier table.
We prove that this construction preserves local RoPE exactly, leaves the attention comparison between two fixed tokens unchanged when other units are inserted or reordered, and does not create new relative rotations merely because more units are added.
In a character-level Tiny Shakespeare diagnostic, all-token validation perplexity remains approximately constant from contexts of 256 to 4096 characters.
A second diagnostic shows that content-based addressing can retrieve and use information from multiple serialized facts.
These are controlled shallow experiments, not scale benchmarks, but they support a direct prescription: use position to address locally and content to address across units.
\end{abstract}


\section{Introduction}
\label{sec:intro}

Long-context Transformers usually assign every token one global position.
At evaluation lengths beyond those used in training, this creates a positional train-test mismatch.
RoPE turns a position difference into a relative rotation~\citep{su2021roformer}.
Attention with Linear Biases (ALiBi) downweights tokens according to distance~\citep{press2022alibi}.
Position interpolation and YaRN rescale RoPE coordinates or frequencies~\citep{chen2023positioninterpolation,peng2023yarn}.
Randomized positions expose training to a broader offset distribution~\citep{ruoss2023randomized}.
These methods can improve behavior beyond the training length, but they still derive each token's rotation from its location in one serialized sequence.

We separate two jobs that global position currently performs.
Within a local text unit, position is the correct address: characters and tokens have an order, and nearby offsets recur across examples.
Across completed units, the unit's integer position in the serialized sequence need not determine its attention rotation.
Using that position makes the rotation between two fixed units depend on how many unrelated tokens were serialized between them.
It also creates new relative rotations solely because more units were placed in context.

We propose \emph{content-based addressing}.
Each token has a bounded position within a unit, such as a line, sentence, paragraph, chunk, or record.
Each completed unit receives an address $R_e$, an orthogonal matrix computed by a shared neural map from its content.
Attention combines the local position and unit address in the same query/key rotation used by RoPE.
The resulting token address is $A(i,e)=R_eR_i$: $R_i$ captures local order within the unit, and $R_e$ is the content-based address of unit $e$.
There is no global distance penalty or global unit counter.
Causal order still determines which units are available, while the attention score determines which available content is relevant.

The paper makes three contributions.
First, it gives a small modification of rotary attention that computes unit addresses from content and can be added to an existing RoPE model.
Second, \Cref{thm:context-invariance} proves that adding units does not create new relative rotations solely because the context is longer.
Third, two controlled diagnostics test complementary properties: stable language-model behavior as more text units are added, and retrieval from content-addressed units when the prediction depends on one of several serialized facts.

Content-based addressing prevents context growth from creating relative rotations unseen during training.
It does not reduce attention cost or softmax competition, and the shared content map must still generalize to new unit contents.

\section{Content-Based Addresses}
\label{sec:method}

Consider a token sequence partitioned into units $e$.
For even $d$, token $i$ in unit $e$ has input embedding $x_{e,i}\in\R^d$, where $i\in\{0,\ldots,L-1\}$ is a bounded local position and $L$ is the maximum unit length.
Let $R_i:=\Rot^i\in\Od$ denote the usual block-diagonal RoPE rotation at local position $i$, where $\mathrm O(d)$ is the group of $d\times d$ orthogonal matrices.

For a completed unit, we first form a content descriptor
\begin{equation}
z_e = \frac{1}{|e|}\sum_{i\in e}R_i x_{e,i}.
\label{eq:pool}
\end{equation}
This descriptor is sensitive to token order because each embedding is rotated according to its local position before pooling.
For example, when ``man'' and ``dog'' have different embeddings, swapping them changes the descriptor because each embedding receives a different local-position rotation.
The learned map can therefore assign the two sentences different unit addresses.
A shared map produces one angle per rotary block,
\begin{equation}
\theta_e=W\,\mathrm{LN}(z_e),
\qquad
R_e=\varphi(z_e):=\operatorname{diag}_b R_2(\theta_{e,b}),
\label{eq:content-address}
\end{equation}
where $\mathrm{LN}$ is LayerNorm and $R_2(\theta)$ is a two-dimensional rotation.
The address of token $(i,e)$ is therefore
\begin{equation}
A(i,e)=R_eR_i.
\label{eq:address}
\end{equation}
Because the implementation uses rotations in the same fixed two-dimensional blocks, the two factors commute and can be implemented by adding their angle vectors.

Queries and keys are rotated exactly as in RoPE:
\begin{equation}
\widehat q_{j,e'}=A(j,e')q_{j,e'},
\qquad
\widehat k_{i,e}=A(i,e)k_{i,e}.
\end{equation}
Their score contains the relative operator
\begin{equation}
\widehat q_{j,e'}^\top\widehat k_{i,e}
=q_{j,e'}^\top
\underbrace{A(j,e')^{-1}A(i,e)}_{P_{(i,e)\to(j,e')}}
k_{i,e}.
\label{eq:score}
\end{equation}
During causal generation, only completed units receive content addresses.
The final incomplete unit uses $R_e=I$ because its complete content descriptor is not yet available.
Let $\mathcal Z\subseteq\R^d$ denote the space of content descriptors $z_e$.

\begin{theorem}[Context-size invariance of content-based addressing]
\label{thm:context-invariance}
Let every unit have local positions in $\{0,\ldots,L-1\}$, let $\varphi:\mathcal Z\to\Od$ be any shared map, set $R_e=\varphi(z_e)$ and $R_i=\Rot^i$, and define $A(i,e)=R_eR_i$.
Then:
\begin{enumerate}
\item Within one unit, the relative operator is ordinary local RoPE:
\[
P_{(i,e)\to(j,e)}=R_j^{-1}R_i=\Rot^{i-j}.
\]
\item Inserting, deleting, or reordering other units does not change the relative operator between two fixed tokens.
\item For any number of units, every relative operator belongs to the fixed family
\[
\mathcal P_{\mathrm{content}}
=\left\{
R_j^{-1}\varphi(z')^{-1}\varphi(z)R_i:
i,j\in\{0,\ldots,L-1\},\ z,z'\in\mathcal Z
\right\},
\]
which does not depend on the number or serialized locations of the units.
By contrast, global RoPE on a serialized context of length $T$ produces the family
$\{\Rot^\delta:|\delta|\leq T-1\}$, which expands with $T$.
\end{enumerate}
\end{theorem}

\begin{proof}
Substitution in \Cref{eq:score} gives
$P=R_j^{-1}R_{e'}^{-1}R_eR_i=R_j^{-1}\varphi(z_{e'})^{-1}\varphi(z_e)R_i$.
When $e=e'$, the content factors cancel, giving $\Rot^{i-j}$.
The expression contains only the two tokens, so other units and their serialized locations cannot affect it.
All possible values lie in the displayed family, whose definition contains $L$ and $\mathcal Z$, but not the number of units.
The global-RoPE statement follows because serialized position differences range from $-(T-1)$ to $T-1$.
\end{proof}

The theorem does not claim that the map will generalize to arbitrary content absent from training.
It states a different guarantee: increasing the number of units does not itself create a new coordinate range.
The shared map also has a parameter count independent of the number of units, unlike a learned identifier table.

\section{Controlled Diagnostics}
\label{sec:experiments}

The diagnostics ask whether the construction can be implemented and used inside standard Transformer attention.
They are single-run, shallow experiments, not comparisons among tuned long-context systems.

\paragraph{Plain-text context growth.}
We train four character-level causal Transformers on Tiny Shakespeare using 256-character chunks and evaluate the same trained models with windows from 256 to 4096 characters.
Training loss is computed at every token.
We report perplexity over all tokens.
All models have four layers, four heads, and embedding dimension 128.
Continuous RoPE uses one global position counter.
ALiBi adds a linear distance bias to each attention head.
The random-address and content-addressed models share the same causal implementation and differ in how they produce line addresses.
The random-address model samples an independent angle for each completed line on every forward pass.
The content-addressed model instead computes each completed line's angle using \Cref{eq:pool,eq:content-address}.
Both models give the final incomplete line zero angle offset, so $R_e=I$.

\begin{table}[t]
\centering
\caption{All-token validation perplexity after training on 256-character Tiny Shakespeare chunks.}
\label{tab:shakespeare}
\small
\setlength{\tabcolsep}{4.5pt}
\begin{tabular}{@{}lccccc@{}}
\toprule
& \multicolumn{5}{c}{\textbf{Evaluation context}} \\
\cmidrule(l){2-6}
\textbf{Addressing scheme} & 256 & 512 & 1024 & 2048 & 4096 \\
\midrule
Content-based address & 4.81 & 4.74 & 4.85 & 4.83 & 4.86 \\
Random line address & 4.89 & 4.79 & 4.90 & 4.87 & 4.91 \\
ALiBi & 4.96 & 4.94 & 4.96 & 4.93 & 4.93 \\
Continuous RoPE & 4.70 & 5.99 & 8.23 & 10.71 & 13.36 \\
\bottomrule
\end{tabular}
\end{table}

\Cref{thm:context-invariance} removes degradation caused by relative rotations unseen during training.
\Cref{tab:shakespeare} shows that the content-addressed model remains approximately constant as context grows to 16 times the training length: all-token perplexity changes from 4.81 to 4.86, while continuous RoPE changes from 4.70 to 13.36.
The content-addressed and random-address models use the same causal implementation.
The content-addressed model has lower all-token perplexity than the random-address model at every evaluated context length in this run.
A flat curve does not by itself show how much distant information the model uses.

\paragraph{Retrieval from serialized facts.}
A second diagnostic tests retrieval from an otherwise unchanged serialized input.
A three-layer character model reads facts serialized with role markers and entity strings, then generates a sentence stating the target fact.
Training examples contain one fact, while evaluation uses either the true fact alone or the true fact plus four distractors.
The \texttt{chain2} facts contain two entities, and the \texttt{chain3} facts contain three.
Hit@5 requires every generated entity-name character to be among the model's five highest-probability predictions.
The baselines use continuous RoPE or replace it with ALiBi using slope 0.004 or 0.25.
The content-addressed model resets local position at each fact and text boundary and computes a content address for every completed fact using the same pooling and projection as above.
All four models receive the same serialized tokens in the same order.
They differ only in how position or content changes the attention scores.

\begin{table}[t]
\centering
\caption{Test hit@5 and perplexity (PPL) on target entity-name characters. Same-entities distractors change the relation. Same-relation distractors change the entities.}
\label{tab:segments}
\scriptsize
\setlength{\tabcolsep}{2.1pt}
\begin{tabular}{@{}llrrrrrrrr@{}}
\toprule
& & \multicolumn{4}{c}{\textbf{hit@5}} & \multicolumn{4}{c}{\textbf{PPL}} \\
\cmidrule(lr){3-6}\cmidrule(l){7-10}
\textbf{Prefix} & \textbf{Task} & RoPE & \shortstack{ALiBi\\(.004)} & \shortstack{ALiBi\\(.25)} & Content & RoPE & \shortstack{ALiBi\\(.004)} & \shortstack{ALiBi\\(.25)} & Content \\
\midrule
Single fact & chain2 & \textbf{.916} & .029 & .653 & .898 & \textbf{1.70} & 59.05 & 3.20 & \textbf{1.70} \\
Single fact & chain3 & .728 & .000 & .382 & \textbf{.800} & 2.23 & 58.36 & 3.60 & \textbf{2.01} \\
Same entities & chain2 & .964 & .102 & .628 & \textbf{.970} & 1.43 & 22.86 & 3.57 & \textbf{1.36} \\
Same entities & chain3 & .603 & .000 & .245 & \textbf{.878} & 3.03 & 31.33 & 4.66 & \textbf{1.68} \\
Same relation & chain2 & .208 & .010 & .015 & \textbf{.523} & 6.32 & 46.70 & 384.65 & \textbf{3.83} \\
Same relation & chain3 & .016 & .000 & .002 & \textbf{.190} & 26.38 & 50.41 & 375.92 & \textbf{5.51} \\
\bottomrule
\end{tabular}
\end{table}

\Cref{tab:segments} shows that the content-addressed model has higher hit@5 in five of six test rows and ties or has lower perplexity in all six.
Continuous RoPE has higher hit@5 only in the single-fact chain2 row.
At distance 100, the attention bias is $-0.4$ with slope 0.004 and $-25$ with slope 0.25.
The stronger ALiBi slope performs better than the weaker slope on the single-fact and same-entities rows, but both remain below content-based addressing in every multi-fact row.
Thus, a flat context-growth curve does not establish retrieval, and we report ALiBi as a control rather than a tuned retrieval comparison.
The multi-fact conditions require the model to retrieve the relevant fact from the serialized prefix and use it to generate the entity name.
They show that content-based addresses can support retrieval from the same serialized input without changing the Transformer architecture.

\section{Related Work and Discussion}
\label{sec:discussion}

Contextual Position Encoding (CoPE) makes position increments depend on context, allowing the model to count selected tokens or higher-level units~\citep{golovneva2024cope}.
PaTH Attention accumulates content-dependent Householder transformations along the token sequence~\citep{yang2025path}.
Both retain an address obtained by accumulating changes through intervening tokens.
Our construction instead gives every completed unit a direct content-based address, so adding unrelated intervening units does not change the attention comparison between two fixed tokens.

The pooling and projection in \Cref{eq:pool,eq:content-address} add $R_e$ to existing RoPE without changing attention heads or Transformer blocks.
Appendix~\ref{app:details} gives implementation details.
\Cref{tab:timing} reports measured training overhead from 10.3\% at 0.8M parameters to 2.0\% at 680M parameters.

A unit can be a completed sentence, paragraph, chapter, document, or book, with $R_e$ computed from its content.
Our experiments use lines and serialized facts.
Causal order preserves which units precede the query and \Cref{eq:pool} preserves internal token order, but $R_e$ does not encode ordinal distance between units.
Attention remains finite, the shared map must generalize to new content, and the model designer must choose unit boundaries.
Appendix~\ref{app:continuity} gives a sufficient condition under which attention scores vary continuously with content descriptors.
In the same-relation evaluation, different entity content can give facts different addresses, but one-fact training never teaches the model to distinguish several facts that share one relation.
The next tests are multi-fact training and fine-tuning this address computation in an open-weight RoPE model.

\section{Conclusion}

Long context need not use one growing positional range.
Content-based addressing preserves local RoPE and gives completed units content-based addresses, so adding units creates no new relative rotations.
Our diagnostics implement this construction in standard Transformer attention.

\bibliographystyle{apalike}
\bibliography{references}

\appendix

\section{Experimental Details}
\label{app:details}

\paragraph{Implementation.}
Zero-initializing the weights of the final angle projection gives $R_e=I$ initially, preserving the original model before fine-tuning.

\subsection{Tiny Shakespeare}

All four models are character-level causal Transformers with a 65-token vocabulary, four layers, four attention heads, embedding dimension 128, pre-norm residual blocks, Gaussian error linear unit (GELU) activations, and scaled dot-product attention.
Parameter counts range from approximately 810K to 819K.
Training uses one million characters with a 90/10 train/validation split, 256-character chunks, 5,000 iterations, batch size 64, AdamW with learning rate $3\times10^{-4}$, a cosine learning-rate schedule, weight decay 0.1, and gradient clipping at 1.0.
Cross-entropy loss is computed at every token.
Validation perplexity is computed over all tokens in windows of 256, 512, 1024, 2048, and 4096 characters.
The ALiBi model uses slopes $2^{-2}$, $2^{-4}$, $2^{-6}$, and $2^{-8}$ across its four heads.

Newlines determine unit boundaries.
The two addressed conditions use the same causal implementation and differ only in how they produce line addresses.
The random-address condition samples an independent angle vector for each completed line on every forward pass, with entries drawn from $\mathcal N(0,1)$.
The content-addressed condition rotates input token embeddings by within-line RoPE, uses a vectorized scatter-add to sum them by line, divides each sum by the line length, and applies LayerNorm followed by a linear map to $d/2$ angles.
All operations are vectorized.
The final incomplete line uses zero content angle, so $R_e=I$.
Every prediction uses only preceding text.

\subsection{Retrieval from Serialized Facts}

\paragraph{Data and prediction task.}
The data contain five domains: family, work, military, school, and church.
Each domain has six ordered roles.
For example, the family roles run from great-grandson through son and father to great-grandfather.
Each domain receives a random permutation of 3,000 lowercase training names of length 2--8.
Consecutive pairs form \texttt{chain2} facts, and consecutive triples form \texttt{chain3} facts.
A \texttt{chain2} example gives one entity and requires the model to generate the other.
A \texttt{chain3} example gives one entity and requires it to generate two.
The data contain 25 \texttt{chain2} relations, 20 \texttt{chain3} relations, and 134,935 facts.
Each fact can be stated with twelve text templates.
Each \texttt{chain2} fact yields two directed prediction tasks, and each \texttt{chain3} fact yields six ordered prediction tasks.
Together, these choices give 6,116,520 sentences.

Each training example contains one role-marked serialized fact followed by one randomly chosen text rendering.
For example, a two-entity prefix and sentence can have the form
\texttt{<son> adam <father> brian | if adam is the son, the father is brian.}
The entity order in the fact prefix is shuffled, so the model must use the role markers rather than prefix order.
Training loss is computed over every token in the sequence, while the reported metrics cover only target entity-name characters.
The text rendering is the current incomplete unit and uses $R_e=I$; content-derived addresses are assigned only to completed facts in its prefix.
During evaluation, the text template is supplied character by character, and only the entity-name characters are scored.
Hit@5 requires every scored character to be among the model's five highest-probability predictions.
Perplexity is computed over those same characters.
Rows labeled train use names seen during training, while rows labeled test use 100 disjoint names.

\paragraph{Evaluation conditions.}
The single-fact condition supplies only the correct fact.
The same-entities condition adds four facts containing the same entity names under different relations, so the model must select the relation requested by the sentence.
The same-relation condition instead adds four facts with the same relation and different entities, so the model must use the given entity to select the correct fact.
The four distractors are sampled uniformly from the facts that satisfy each condition.
Training contains only one fact per example.
It therefore never requires the model to choose among several facts with the same relation.
Content-based addressing gives those facts different addresses, but learning to use those addresses for this choice requires multi-fact training.

\paragraph{Models and training.}
All four models are three-layer, single-head causal Transformers with embedding dimension 60, softmax attention, batch size 32, learning rate $5\times10^{-4}$, and 100K training iterations.
Continuous RoPE uses one position counter across the fact prefix and sentence.
The content-addressed model sees the same serialized tokens, resets local position at each fact and text boundary, and assigns every completed fact a content-based address.
The two ALiBi models replace RoPE with slopes $2^{-8}\approx0.004$ and $2^{-2}=0.25$.
Because this diagnostic uses one attention head, each ALiBi model uses one slope rather than the set of slopes distributed across heads in a multi-head model.
The two runs use the weakest and strongest slopes in the four-head ALiBi schedule.

\Cref{tab:segments-full} gives the complete train and test results.
The content-addressed model leads on hit@5 in every multi-fact test row and five of six test rows overall.
The same-relation rows are the hardest for every model; content-based addressing reaches .523 hit@5 on chain2 and .190 on chain3.
On the single-fact train rows, the stronger ALiBi slope reaches .959 hit@5 on chain2 and .825 on chain3, while the weaker slope reaches .112 and .008.

\begin{table}[ht]
\centering
\caption{Complete results for retrieval from serialized facts. Higher hit@5 and lower perplexity are better. The best value in each metric and row is bold.}
\label{tab:segments-full}
\scriptsize
\setlength{\tabcolsep}{2.1pt}
\resizebox{\textwidth}{!}{%
\begin{tabular}{@{}lllrrrrrrrr@{}}
\toprule
& & & \multicolumn{4}{c}{\textbf{hit@5}} & \multicolumn{4}{c}{\textbf{PPL}} \\
\cmidrule(lr){4-7}\cmidrule(l){8-11}
\textbf{Prefix} & \textbf{Fact} & \textbf{Set} & RoPE & \shortstack{ALiBi\\(.004)} & \shortstack{ALiBi\\(.25)} & Content & RoPE & \shortstack{ALiBi\\(.004)} & \shortstack{ALiBi\\(.25)} & Content \\
\midrule
Single fact & chain2 & train & \textbf{1.000} & .112 & .959 & \textbf{1.000} & 1.13 & 15.22 & 1.47 & \textbf{1.09} \\
Single fact & chain2 & test  & \textbf{.916} & .029 & .653 & .898 & \textbf{1.70} & 59.05 & 3.20 & \textbf{1.70} \\
Single fact & chain3 & train & .957 & .008 & .825 & \textbf{1.000} & 1.28 & 14.44 & 1.49 & \textbf{1.22} \\
Single fact & chain3 & test  & .728 & .000 & .382 & \textbf{.800} & 2.23 & 58.36 & 3.60 & \textbf{2.01} \\
\midrule
Same entities & chain2 & train & \textbf{1.000} & .250 & .871 & \textbf{1.000} & 1.17 & 10.63 & 1.80 & \textbf{1.08} \\
Same entities & chain2 & test  & .964 & .102 & .628 & \textbf{.970} & 1.43 & 22.86 & 3.57 & \textbf{1.36} \\
Same entities & chain3 & train & .696 & .014 & .510 & \textbf{1.000} & 2.17 & 15.24 & 2.14 & \textbf{1.22} \\
Same entities & chain3 & test  & .603 & .000 & .245 & \textbf{.878} & 3.03 & 31.33 & 4.66 & \textbf{1.68} \\
\midrule
Same relation & chain2 & train & .323 & .029 & .034 & \textbf{.665} & 2.90 & 23.37 & 21.74 & \textbf{2.13} \\
Same relation & chain2 & test  & .208 & .010 & .015 & \textbf{.523} & 6.32 & 46.70 & 384.65 & \textbf{3.83} \\
Same relation & chain3 & train & .042 & .001 & .003 & \textbf{.366} & 7.88 & 28.00 & 16.13 & \textbf{2.17} \\
Same relation & chain3 & test  & .016 & .000 & .002 & \textbf{.190} & 26.38 & 50.41 & 375.92 & \textbf{5.51} \\
\bottomrule
\end{tabular}%
}
\end{table}

\section{Timing Benchmark}
\label{app:timing}

We measured one complete training step for the continuous-RoPE baseline and content-addressed model on an NVIDIA Tesla T4 using PyTorch 2.6.0, CUDA 12.4, and Python 3.12.
Context length was fixed at 1024.
Each result averages ten timed steps after one warmup step, with CUDA synchronization around the timed region.
The timing benchmark uses batch size 8 at 0.8M parameters rather than the training batch size of 64.
Batch size was reduced further for larger models to fit device memory.

\begin{table}[h]
\centering
\caption{Training-step timing at context length 1024. The baseline time includes the forward pass, backward pass, and optimizer update. Overhead is the content-addressed model's increase relative to the baseline.}
\label{tab:timing}
\begin{tabular}{@{}lrrrrrr@{}}
\toprule
\textbf{Parameters} & $d$ & \textbf{Layers} & \textbf{Heads} & \textbf{Batch} & \textbf{RoPE ms} & \textbf{Overhead} \\
\midrule
0.8M & 128 & 4 & 4 & 8 & 42 & +10.3\% \\
85M & 768 & 12 & 12 & 4 & 702 & +4.4\% \\
302M & 1024 & 24 & 16 & 2 & 1215 & +3.5\% \\
680M & 1536 & 24 & 16 & 1 & 1428 & +2.0\% \\
\bottomrule
\end{tabular}
\end{table}

The content-addressed model adds vectorized boundary detection, local rotation before pooling, scatter-add mean pooling, and a LayerNorm followed by a $d$-to-$d/2$ linear projection.
These operations require $O(Td)$ work for fixed hidden size.
Transformer training includes attention and feed-forward network work of order $O(T^2d+Td^2)$.

\section{Continuity of Content-Derived Scores}
\label{app:continuity}

Context-size invariance removes growth caused by the number of units, but it does not by itself constrain how addresses vary with content.
The following standard bound states one sufficient condition.

\begin{proposition}[Content-derived addresses change scores continuously]
Let $(\mathcal Z,\rho)$ be a metric space and let $\varphi:\mathcal Z\to\Od$ be $K$-Lipschitz in operator norm.
For unit-norm $q,k$, define
\[
F(z,z';i,j)=q^\top R_j^{-1}\varphi(z')^{-1}\varphi(z)R_i k.
\]
Then
\[
|F(z,z';i,j)-F(u,u';i,j)|
\leq K\bigl(\rho(z,u)+\rho(z',u')\bigr).
\]
\end{proposition}

\begin{proof}
Insert and subtract
$q^\top R_j^{-1}\varphi(z')^{-1}\varphi(u)R_i k$.
The triangle inequality, orthogonality of all surrounding factors, and
$\|\varphi(z')^{-1}-\varphi(u')^{-1}\|_{\mathrm{op}}=\|\varphi(z')-\varphi(u')\|_{\mathrm{op}}$
give the result.
\end{proof}

This stability bound does not require or guarantee that distinct descriptors receive distinct addresses.

\end{document}